\documentclass[11pt]{article}
\usepackage[numbers]{natbib}
\usepackage{amsmath}
\usepackage{amssymb}
\usepackage{amsthm}
\usepackage{xcolor}
\usepackage{graphicx}
\usepackage{hyperref}
\usepackage{algorithm}
\usepackage{algpseudocode}
\usepackage{enumitem}
\usepackage{array}
\usepackage{booktabs}
\usepackage{longtable}
\usepackage[margin=1in]{geometry}
\hypersetup{colorlinks=true,citecolor=blue,linkcolor=blue,urlcolor=blue}
\allowdisplaybreaks
\newcommand{\paperTitle}{Tight Worst-Case Bounds for the Smallest Eigenvalue of ReLU NTK Gram Matrices}
\newcommand{\paperAuthor}{Zhao Song\thanks{\texttt{magic.linuxkde@gmail.com}.}}

\DeclareMathOperator{\E}{\mathbb{E}}

\theoremstyle{plain}
\newtheorem{theorem}{Theorem}[section]
\newtheorem{lemma}[theorem]{Lemma}

\theoremstyle{definition}

\begin{document}
\date{}
\title{\paperTitle}
\author{\paperAuthor}
\maketitle
\begin{abstract}
For $n$ unit vectors $x_1,\ldots,x_n \in \mathbb{R}^d$, we study the continuous ReLU derivative Gram matrix $H$, whose entries are obtained by averaging pairwise gated inner products over a standard Gaussian direction. Writing
$
\Delta_\pm
:=
\min_{i \neq j}
\min\{
\|x_i-x_j\|_2,
\|x_i+x_j\|_2
\}
$
for their projective separation, we prove the universal dimension-free lower bound
$
\lambda_{\min}(H)
=
\Omega(
\Delta_\pm/\sqrt{\log n}
)
$.
Conversely, we construct worst-case families satisfying the matching upper bound
$
\lambda_{\min}(H)
=
O(
\Delta_\pm/\sqrt{\log n}
)
$,
showing that this rate is tight up to universal constants.

\end{abstract}

\section{Introduction}

Following Jacot, Gabriel, and Hongler~\cite{jgh18}, the neural tangent kernel (NTK) of a scalar-valued predictor $f_\theta$ is the function
$
\Theta_\theta(x,x')
:=
\langle \nabla_\theta f_\theta(x),\nabla_\theta f_\theta(x')\rangle .
$
Thus the NTK is a function on pairs of inputs. Given a data set $X=\{x_1,\ldots,x_n\}$, evaluating this function on the observed pairs produces the neural tangent Gram matrix, or NTK Gram matrix,
$
K_\theta(X)_{ij}
:=
\Theta_\theta(x_i,x_j),
$ $
1\leq i,j\leq n.
$ 
The kernel is defined on the ambient input space, whereas its Gram matrix is the finite, data-dependent restriction $K_\theta(X)\in\mathbb{R}^{n\times n}$. By construction, $K_\theta(X)$ is positive semidefinite. At a random finite-width initialization it is generally random, while its infinite-width limit at initialization is typically deterministic. The smallest eigenvalue of the Gram matrix quantifies the conditioning of the corresponding linearized training problem and governs the convergence rates obtained in training analyses~\cite{dzps19}.

Consider the two-layer ReLU parameterization
$
f_W(x)
:=
\frac{1}{\sqrt{m}}
\sum_{r=1}^m
a_r\sigma(w_r^\top x),
$ $
\sigma(t):=\max\{t,0\},
$
with fixed output signs $a_r\in\{-1,1\}$ and trainable hidden weights $w_r$. Differentiating with respect to the hidden weights gives
$
(H^{\mathrm{dis}})_{ij}
:=
\frac{1}{m}
\sum_{r=1}^m
x_i^\top x_j
\mathbf{1}\{w_r^\top x_i\geq 0\}
\mathbf{1}\{w_r^\top x_j\geq 0\}.
$
In this hidden-weight-only model, $H^{\mathrm{dis}}$ is exactly the finite-width NTK Gram matrix. For independent Gaussian initialization of the $w_r$'s, the matrix $H=\E[H^{\mathrm{dis}}]$ studied in this paper is its deterministic population, or infinite-width, limit. If the output weights are also trained, the full two-layer NTK contains an additional activation-covariance term arising from the output-weight gradients; in that broader parameterization, $H$ is the hidden-weight derivative component rather than the full NTK.

For unit-normalized data, we measure geometric nondegeneracy by the projective separation $\Delta_\pm$, which rules out both nearly coincident and nearly antipodal pairs. The quantity $\Delta_{\pm}:=\min_{i \neq j}\min\{\|x_i-x_j\|_2,\|x_i+x_j\|_2\}$ is the sign-invariant analogue of the Euclidean data-separation parameter used in multilayer neural network training analyses~\cite{als19}. Under the normalization in that work, which gives every input a common positive coordinate, the two parameters coincide.

Both $\lambda_{\min}(H)$ and $\Delta_{\pm}$ play a crucial role in neural network optimization. It is natural to establish a sharp relationship between them. In the work \cite{son19}, they explicitly ask the following question as an open problem:
\begin{center}
{\it What is the tight relationship between $\lambda_{\min}(H)$ and $\Delta_{\pm}$?}
\end{center}

Quantitative conditioning of ReLU Gram and NTK matrices has been studied under several related hypotheses. Panigrahi, Shetty, and Goyal~\cite[Theorem~4.2 and Appendix Theorem~L.2]{psg20} use the projective sine separation $\delta_{\mathrm{PSG}}:=\min_{i \neq j}\sqrt{1-\langle x_i,x_j\rangle^2}$, which satisfies $\Delta_\pm/\sqrt{2}\leq\delta_{\mathrm{PSG}}\leq\Delta_\pm$, and state that, if $m\gtrsim n^4\delta_{\mathrm{PSG}}^{-3}\log^4 n$, then the randomly initialized finite-width matrix satisfies $\lambda_{\min}(H^{\mathrm{dis}})\gtrsim(\delta_{\mathrm{PSG}}/\log n)^{3/2}$ with high probability. Karhadkar, Murray, and Mont{\'u}far~\cite[Theorem~1 and Lemmas~5 and~7]{kmm24} use the separation $\Delta_\pm$. For $d\geq 3$, their population results, combined with the Schur product theorem, imply
$\lambda_{\min}(H)\gtrsim\delta_0^2(1+d\log(1/\delta_0)/\log d)^{-3}$,
where $\delta_0:=\min\{\Delta_\pm,1/2\}$; their Theorem~1 gives the related finite-width full-NTK result under a width condition. For fixed dimension, Liu, Mao, and Xu~\cite[Theorem~1 and Eq.~(2.12)]{lmx25} studied the same kernel as the $k=s=1$ stiffness matrix of a shallow ReLU network and proved $\lambda_{\min}(H)\gtrsim_d\Delta_\pm$, with an implicit constant that may depend on $d$.

Oymak and Soltanolkotabi~\cite[Corollary~I.2]{os20} proved
\[
\lambda_{\min}(H)
\geq 
\frac{\Delta_\pm}{100n^2} .
\]
Their result gives a dimension-independent certificate of strict positive definiteness, but the guaranteed conditioning deteriorates quadratically with the sample size.

We improve this estimate to the universal lower bound
\[
\lambda_{\min}(H)
=
\Omega(
\frac{\Delta_\pm}{\sqrt{\log n}}
)
\]
for every projectively separated data set.  We also give a matching existential construction: balanced binary linear codes embedded in a spherical cap produce families with $n=2^{\Theta(d)}$ and minimum eigenvalue 
\[
\lambda_{\min}(H) = O( \frac{ \Delta_\pm}{\sqrt{\log n} } ).
\]
Thus the new dependence is sharp up to universal constants in the dimension-free worst case, although individual configurations may have substantially larger spectral gaps.

The lower bound follows from the positive Hadamard-power expansion of the arcsine kernel and a diagonally dominant high-degree tail.  The upper construction is analyzed by Fourier diagonalization over the code together with a lazy coordinate walk.

\begin{theorem}[informal, combination of Lemma~\ref{lem:cts_lower_bound} and Lemma~\ref{lem:cts_upper_bound}]
\label{thm:sharp_cts_eigenvalue}
Let $n \ge 2$, let $x_1,\ldots,x_n \in \mathbb{R}^d$ be unit vectors, and define
$
H_{ij}
:=
\E_w[
x_i^\top x_j
\mathbf{1}\{w^\top x_i \ge 0\}
\mathbf{1}\{w^\top x_j \ge 0\}
]$,
$
w \sim \mathcal{N}(0,I_d).
$
Define
$
\Delta_\pm
:=
\min_{i \neq j}
\min\{
\|x_i-x_j\|_2,
\|x_i+x_j\|_2
\}.
$ For every such collection, the universal lower bound
$
\lambda_{\min}(H)
=
\Omega(
\Delta_\pm/\sqrt{\log n}
)
$
holds. Conversely, a worst-case construction produces families with $n=2^{\Theta(d)}$ for which
$
\lambda_{\min}(H)
=
O(
\Delta_\pm/\sqrt{\log n}
)
$,
showing that the bound is tight up to universal constants.
\end{theorem}

\section{Proofs of the Eigenvalue Bounds}
\label{sec:cts_eigenvalue_bounds}

Let $x_1,\ldots,x_n \in \mathbb{R}^d$ be unit vectors and let $w \sim \mathcal{N}(0,I_d)$. Define
\[
H_{ij}
:=
\E_w[
x_i^\top x_j
\mathbf{1}\{w^\top x_i \ge 0\}
\mathbf{1}\{w^\top x_j \ge 0\}
]
\]
and
\[
\Delta_\pm
:=
\min_{i \neq j}
\min\{
\|x_i-x_j\|_2,
\|x_i+x_j\|_2
\}.
\]

\begin{lemma}[Universal lower bound]
\label{lem:cts_lower_bound}
Suppose that $n \ge 2$ and $\Delta_\pm>0$. Then
\[
\lambda_{\min}(H)
\ge
\frac{\Delta_\pm}{70\sqrt{\log n}}.
\]
\end{lemma}

\begin{proof}

Let $G \in \mathbb{R}^{n \times n}$ be the data Gram matrix, so that $G_{ij}=x_i^\top x_j$, and write $\rho_{ij}:=G_{ij}$. Since the data points have unit norm, for every $i \neq j$,
\[
\min\{
\|x_i-x_j\|_2^2,
\|x_i+x_j\|_2^2
\}
=
2(1-|\rho_{ij}|).
\]
Consequently,
\[
|\rho_{ij}|
\le
1-\frac{\Delta_\pm^2}{2}.
\]

The pair $(w^\top x_i,w^\top x_j)$ is a centered bivariate Gaussian with correlation $\rho_{ij}$. Hence
\[
\Pr[
w^\top x_i \ge 0,
w^\top x_j \ge 0
]
=
\frac{\pi-\arccos(\rho_{ij})}{2\pi},
\]
and therefore
\[
H_{ij}
=
\frac{
\rho_{ij}
(\pi-\arccos(\rho_{ij}))
}{2\pi}
=
\frac{\rho_{ij}}{4}
+
\frac{\rho_{ij}\arcsin(\rho_{ij})}{2\pi}.
\]
Using the absolutely convergent expansion
\[
\arcsin(t)
=
\sum_{k=0}^{\infty}
a_k t^{2k+1},
\qquad
a_k
:=
\frac{\binom{2k}{k}}{4^k(2k+1)},
\]
we obtain the matrix identity
\[
H
=
\frac14 G
+
\frac{1}{2\pi}
\sum_{k=0}^{\infty}
a_k G^{\circ(2k+2)},
\]
where $G^{\circ r}$ denotes the entrywise $r$th power of $G$. By the Schur product theorem~\cite[Theorem~7.5.3]{hj12}, $G$ and every $G^{\circ r}$ are positive semidefinite.

Set
\[
K
:=
\lceil
\frac{2\log n}{\Delta_\pm^2}
\rceil.
\]
For every $k \ge K$ and every $i \neq j$,
\[
|
(G^{\circ(2k+2)})_{ij}
|
\le
(1-\frac{\Delta_\pm^2}{2})^{2k+2}
\le
\exp(-(k+1)\Delta_\pm^2)
\le
\frac{1}{n^2}.
\]
The diagonal entries of $G^{\circ(2k+2)}$ equal $1$. By Gershgorin's circle theorem~\cite[Theorem~6.1.1]{hj12},
\[
G^{\circ(2k+2)}
\succeq
(
1-\frac{n-1}{n^2}
)I_n
\succeq
\frac34 I_n.
\]
All coefficients $a_k$ are positive, so discarding the other positive semidefinite terms yields
\[
H
\succeq
\frac{1}{2\pi}
\sum_{k=K}^{\infty}
a_k G^{\circ(2k+2)}
\succeq
\frac{3}{8\pi}
(
\sum_{k=K}^{\infty}a_k
)I_n.
\]

The standard central-binomial estimate, a consequence of Wallis' inequalities~\cite[p.~71, Eq.~(19)]{cha68}, gives, for every $k \ge 1$,
\[
\frac{\binom{2k}{k}}{4^k}
\ge
\frac{1}{2\sqrt{k}}.
\]
Since $2k+1 \le 3k$,
\[
a_k
\ge
\frac{1}{6k^{3/2}},
\qquad
\sum_{k=K}^{\infty}a_k
\ge
\frac16
\int_K^{\infty}x^{-3/2}dx
=
\frac{1}{3\sqrt{K}}.
\]
It follows that
\[
\lambda_{\min}(H)
\ge
\frac{1}{8\pi\sqrt{K}}.
\]

Finally, $\Delta_\pm^2 \le 2$ for unit vectors and $n \ge 2$. Hence $3\log n \ge 3\log 2>2\ge\Delta_\pm^2$, so
\[
K
\le
\frac{2\log n}{\Delta_\pm^2}+1
\le
\frac{5\log n}{\Delta_\pm^2}.
\]
Substitution gives
\[
\lambda_{\min}(H)
\ge
\frac{\Delta_\pm}{
8\pi\sqrt{5\log n}
}
\ge
\frac{\Delta_\pm}{70\sqrt{\log n}}.
\]
\end{proof}

\begin{lemma}[Upper bound construction]
\label{lem:cts_upper_bound}
There are universal constants $c_0,C_0,C_1,R>0$ such that, for every sufficiently large integer $D$ and every $0<\tau\le 1/4$, there exist
\[
n=2^{\lfloor RD\rfloor}
\]
unit vectors in $\mathbb{R}^{D+1}$ satisfying
\[
c_0\sqrt{\tau}
\le
\Delta_\pm
\le
C_0\sqrt{\tau}
\]
and
\[
\lambda_{\min}(H)
\le
C_1\sqrt{\frac{\tau}{D}}
\le
C_1\frac{\Delta_\pm}{\sqrt{\log n}}.
\]
\end{lemma}

\begin{proof}
For $0<t<1$, let
\[
H_2(t)
:=
-t\log_2 t-(1-t)\log_2(1-t)
\]
be the binary entropy function. Fix constants $0<\alpha<1/2$ and $0<R<1-H_2(\alpha)$, and then fix $0<\beta<1/2$ such that $H_2(\beta)<R$.

For every sufficiently large $D$, there is a binary linear code $C\le\mathbb{F}_2^D$ of dimension
\[
k=\lfloor RD\rfloor
\]
such that every $c\in C\setminus\{0\}$ satisfies
\[
\alpha D
\le
|c|
\le
(1-\alpha)D,
\]
where $|c|$ denotes Hamming weight. Indeed, choose a uniformly random $k$-dimensional subspace. The number of nonzero vectors outside this weight interval is at most $2^{H_2(\alpha)D+1}$, while the probability that a fixed nonzero vector belongs to the random subspace is at most $2^{k-D+1}$. The expected number of such vectors in the subspace is therefore less than $1$ for all sufficiently large $D$.

For each $c\in C$, define
\[
x_c
:=
(
\sqrt{1-\tau},
\sqrt{\frac{\tau}{D}}(-1)^{c_1},
\ldots,
\sqrt{\frac{\tau}{D}}(-1)^{c_D}
).
\]
These are unit vectors, and
\[
x_c^\top x_{c'}
=
1-\frac{2\tau}{D}|c+c'|.
\]
All these correlations are at least $1-2\tau\ge 1/2$, so projective separation equals ordinary separation for this data set. The weight property of $C$ gives
\[
2\sqrt{\alpha\tau}
\le
\Delta_\pm
\le
2\sqrt{\tau}.
\]
Also, the number of vectors is
\[
n=|C|=2^k.
\]

Define
\[
\kappa(t)
:=
\frac{t(\pi-\arccos t)}{2\pi}
\]
and, on $\mathbb{F}_2^D$, define
\[
F(u)
:=
\kappa(1-\frac{2\tau|u|}{D}).
\]
The Gram matrix indexed by $C$ has entries $F(c+c')$ and is therefore a convolution matrix on $C$.

For $b\in\mathbb{F}_2^D$, let
\[
\Lambda_b
:=
\sum_{u\in\mathbb{F}_2^D}
F(u)(-1)^{b^\top u}.
\]
Characters of $C$ are indexed by cosets $A=a+C^\perp$. The eigenvalue corresponding to $A$ is
\[
\lambda_A
=
\sum_{c\in C}
F(c)(-1)^{a^\top c}
=
\frac{1}{|C^\perp|}
\sum_{b\in A}\Lambda_b.
\]
The power expansion of $\kappa$ below shows that every $\Lambda_b$ is nonnegative.

At most
\[
\sum_{j<\beta D}\binom{D}{j}
\le
2^{H_2(\beta)D}
\]
cosets of $C^\perp$ contain a vector of weight less than $\beta D$. Since there are $n=2^{\lfloor RD\rfloor}$ cosets and $H_2(\beta)<R$, at least $n/2$ cosets have every vector of weight at least $\beta D$ when $D$ is sufficiently large. Call these cosets far.

It remains to bound the total Fourier mass at high Hamming weight. Set
\[
L:=\lceil\beta D\rceil,
\qquad
T_\beta
:=
2^{-D}
\sum_{|b|\ge L}\Lambda_b.
\]
The expansion used in the lower bound can be written as
\[
\kappa(t)
=
\frac{t}{4}
+
\sum_{\substack{r\ge 2\\r\text{ even}}}
q_r t^r,
\qquad
0\le q_r\le C r^{-3/2}
\]
for a universal constant $C$. Moreover,
\[
1-\frac{2\tau|u|}{D}
=
1-\tau
+
\frac{\tau}{D}
\sum_{\ell=1}^D(-1)^{u_\ell}.
\]
The right-hand side is the Fourier transform of one step of the lazy walk on $\mathbb{F}_2^D$ that stays put with probability $1-\tau$ and flips a uniformly random coordinate with probability $\tau$. If $Z_r$ is the position after $r$ steps, Fourier inversion gives
\[
2^{-D}
\sum_{|b|\ge L}\Lambda_b
=
\sum_{\substack{r\ge 2\\r\text{ even}}}
q_r\Pr[|Z_r|\ge L].
\]
Here the linear term contributes only at Hamming weights $0$ and $1$.

Let $N_r$ be the number of actual coordinate flips among the $r$ steps. Then $N_r$ is binomial with mean $r\tau$, and $|Z_r|\le N_r$. Markov's inequality gives
\[
\Pr[|Z_r|\ge L]
\le
\min\{1,\frac{r\tau}{L}\}.
\]
Splitting the series at $r_0=L/\tau$ yields
\begin{align*}
T_\beta
\le
C\frac{\tau}{L}
\sum_{2\le r\le r_0}r^{-1/2}
+
C\sum_{r>r_0}r^{-3/2}
\le
C_\beta\sqrt{\frac{\tau}{D}}.
\end{align*}
The first inequality follows from the Fourier-mass identity, $q_r\le Cr^{-3/2}$, and $\Pr[|Z_r|\ge L]\le\min\{1,r\tau/L\}$. The second follows from the integral estimates $\sum_{2\le r\le r_0}r^{-1/2}=O(\sqrt{r_0})$ and $\sum_{r>r_0}r^{-3/2}=O(r_0^{-1/2})$, together with $r_0=L/\tau$ and $L\ge\beta D$.

Let $\mathcal{A}_{\mathrm{far}}$ be the collection of far cosets and let $M=|C^\perp|$. Averaging the corresponding eigenvalues gives
\begin{align*} \lambda_{\min}(H) \le \min_{A\in\mathcal{A}_{\mathrm{far}}}\lambda_A \le \frac{1}{|\mathcal{A}_{\mathrm{far}}|M} \sum_{A\in\mathcal{A}_{\mathrm{far}}} \sum_{b\in A}\Lambda_b \le \frac{2}{2^D} \sum_{|b|\ge L}\Lambda_b = 2T_\beta \le 2C_\beta\sqrt{\frac{\tau}{D}}. \end{align*} 
where the first step follows from the fact that the cosets of $C^\perp$ index the full spectrum, the second step follows from bounding the minimum over $\mathcal{A}_{\mathrm{far}}$ by its average and substituting the formula for $\lambda_A$, and the third step follows from the disjointness of the far cosets inside $\{b:|b|\ge L\}$, the nonnegativity of each $\Lambda_b$, the estimate $|\mathcal{A}_{\mathrm{far}}|M\ge(n/2)(2^D/n)=2^{D-1}$, the definition of $T_\beta$, and the bound on $T_\beta$ above.

The preceding display gives the first inequality below. Moreover, $\Delta_\pm\ge2\sqrt{\alpha\tau}$ gives $\sqrt{\tau/D}\le\Delta_\pm/(2\sqrt{\alpha D})$, while $n=2^{\lfloor RD\rfloor}$ gives $\log n\le RD\log 2$. Hence, after the fixed choices of $\alpha$, $\beta$, and $R$, a universal constant $C_1$ satisfies
\[
\lambda_{\min}(H)
\le
2C_\beta\sqrt{\frac{\tau}{D}}
\le
C_1\frac{\Delta_\pm}{\sqrt{\log n}}.
\]
\end{proof}

\section{Ordinary versus Projective Separation}
\label{sec:ordinary_separation}

Define the ordinary separation by
$
\Delta_-:=\min_{i\neq j}\|x_i-x_j\|_2.
$
Although $\Delta_\pm\leq\Delta_-$, the universal lower bound cannot be strengthened by replacing $\Delta_\pm$ with $\Delta_-$. The proof above needs
\[
|\rho_{ij}|\leq 1-\frac{\Delta_\pm^2}{2},
\]
which forces the off-diagonal entries of the high even Hadamard powers to decay. Ordinary separation controls correlations near $1$, but it does not exclude correlations near $-1$.

A regular hexagon gives an exact obstruction. Let
\[
u_1=(1,0),\qquad
u_2=(1/2,\sqrt{3}/2),\qquad
u_3=(-1/2,\sqrt{3}/2),
\]
so that $u_1-u_2+u_3=0$, and order the six points as
\[
(x_1,\ldots,x_6)=(u_1,u_2,u_3,-u_1,-u_2,-u_3).
\]
This configuration satisfies $\Delta_-=1$. For the derivative feature
\[
\Phi_w(x):=x\mathbf{1}_{\{w^\top x\geq0\}},
\]
we have $\Phi_w(u)-\Phi_w(-u)=u$ for almost every Gaussian $w$. Therefore, for $z=(1,-1,1,-1,1,-1)^\top$,
\[
\sum_{i=1}^6 z_i\Phi_w(x_i)=u_1-u_2+u_3=0
\]
almost surely, and hence
\[
z^\top H z
=
\E_w[\|\sum_{i=1}^6 z_i\Phi_w(x_i)\|_2^2]
=
0.
\]
Since $H\succeq0$, this implies $\lambda_{\min}(H)=0$. Thus no positive universal lower bound depending only on $\Delta_-$ can hold, even when $n=6$ and $\Delta_-=1$.

This obstruction is robust and is not merely an artifact of exact antipodality. For $0<\varepsilon<\pi/6$, rotate the three negative vertices by angle $\varepsilon$ and set
\[
X_\varepsilon
:=
(u_1,u_2,u_3,-R_\varepsilon u_1,-R_\varepsilon u_2,-R_\varepsilon u_3).
\]
The perturbed configuration has no antipodal pair and satisfies
\[
\Delta_-(X_\varepsilon)
=
2\sin(\pi/6-\varepsilon/2)
\longrightarrow
1.
\]
The entries of $H(X_\varepsilon)$ depend continuously on the pairwise inner products, so $\lambda_{\min}(H(X_\varepsilon))\to0$. Consequently, merely excluding exact antipodal pairs still does not yield a positive lower bound in terms of $\Delta_-$ alone.

The upper-bound construction is unaffected. Define
\[
\Delta_+:=\min_{i\neq j}\|x_i+x_j\|_2,
\qquad
\Delta_\pm=\min\{\Delta_-,\Delta_+\}.
\]
For unit vectors,
$
\|x_i+x_j\|_2^2-\|x_i-x_j\|_2^2
=
4x_i^\top x_j.
$
Hence, if all pairwise inner products are nonnegative, then $\Delta_+\geq\Delta_-$ and therefore $\Delta_\pm=\Delta_-$. In particular, the spherical-cap construction above satisfies $x_i^\top x_j\geq1/2$, so its ordinary and projective separations coincide and it gives
\[
\lambda_{\min}(H)
\lesssim
\frac{\Delta_-}{\sqrt{\log n}}.
\]
Under the additional assumption $x_i^\top x_j\geq0$ for all $i\neq j$, the universal lower bound also transfers verbatim from $\Delta_\pm$ to $\Delta_-$. More generally, if $\Delta_+\geq c\Delta_-$ for some $c>0$, then $\Delta_\pm\geq\min\{1,c\}\Delta_-$ and the lower bound transfers with the corresponding constant loss.

\section*{Acknowledgment}

The AI tools such as codex 5.6 and Claude code 5.5 Fable are used for grammar checking and language editing.


\bibliographystyle{alpha}
\bibliography{ref}

\begin{thebibliography}{KMM24}

\bibitem[AZLS19]{als19}
Zeyuan Allen-Zhu, Yuanzhi Li, and Zhao Song.
\newblock A convergence theory for deep learning via over-parameterization.
\newblock In {\em Proceedings of the 36th International Conference on Machine
  Learning}, volume~97 of {\em Proceedings of Machine Learning Research}, pages
  242--252. PMLR, 2019.

\bibitem[Cha68]{cha68}
K.~Chandrasekharan.
\newblock {\em Introduction to Analytic Number Theory}, volume 148 of {\em
  Grundlehren der mathematischen Wissenschaften}.
\newblock Springer-Verlag, Berlin, Heidelberg, 1968.

\bibitem[DZPS19]{dzps19}
Simon~S. Du, Xiyu Zhai, Barnab{\'a}s P{\'o}czos, and Aarti Singh.
\newblock Gradient descent provably optimizes over-parameterized neural
  networks.
\newblock In {\em International Conference on Learning Representations}, 2019.

\bibitem[HJ12]{hj12}
Roger~A. Horn and Charles~R. Johnson.
\newblock {\em Matrix Analysis}.
\newblock Cambridge University Press, Cambridge, 2 edition, 2012.

\bibitem[JGH18]{jgh18}
Arthur Jacot, Franck Gabriel, and Cl{\'e}ment Hongler.
\newblock Neural tangent kernel: Convergence and generalization in neural
  networks.
\newblock In {\em Advances in Neural Information Processing Systems},
  volume~31, pages 8571--8580, 2018.

\bibitem[KMM24]{kmm24}
Kedar Karhadkar, Michael Murray, and Guido Mont{\'u}far.
\newblock Bounds for the smallest eigenvalue of the {NTK} for arbitrary
  spherical data of arbitrary dimension.
\newblock In {\em Advances in Neural Information Processing Systems},
  volume~37, pages 138197--138249, 2024.

\bibitem[LMX25]{lmx25}
Xinliang Liu, Tong Mao, and Jinchao Xu.
\newblock Condition numbers and eigenvalue spectra of shallow networks on
  spheres, 2025.

\bibitem[OS20]{os20}
Samet Oymak and Mahdi Soltanolkotabi.
\newblock Toward moderate overparameterization: Global convergence guarantees
  for training shallow neural networks.
\newblock {\em IEEE Journal on Selected Areas in Information Theory},
  1(1):84--105, 2020.

\bibitem[PSG20]{psg20}
Abhishek Panigrahi, Abhishek Shetty, and Navin Goyal.
\newblock Effect of activation functions on the training of overparametrized
  neural nets.
\newblock In {\em International Conference on Learning Representations}, 2020.

\bibitem[Son19]{son19}
Zhao Song.
\newblock {\em Matrix Theory: Optimization, Concentration and Algorithms}.
\newblock PhD thesis, The University of Texas at Austin, August 2019.

\end{thebibliography}
\end{document}